\documentclass[11pt,letterpaper]{article}

 \usepackage{booktabs}       

\usepackage{tcs}

\usepackage{mysymbols} 
\usepackage{mathdots}

\usepackage{cleveref}
\def\showauthornotes{1}
\ifnum\showauthornotes=1
\newcommand{\anote}[1]{{\sf\color{orange}{ [Ainesh: #1] }}}
\newcommand{\gnote}[1]{{\sf\color{teal}{ [Goutham: #1] }}}

\else
\newcommand{\anote}[1]{}
\newcommand{\gnote}[1]{}
\fi

\newcommand{\norm}[1]{\ensuremath{\lVert #1 \rVert}}
\newcommand{\EE}{\mathbb{E}}

\title{Tensor Decompositions Meet Control Theory: \\Learning General Mixtures of Linear Dynamical Systems}

\author{
Ainesh Bakshi\thanks{Supported by Ankur Moitra's ONR grant} \\
\texttt{ainesh@mit.edu} \\
MIT
\and
Allen Liu\thanks{Supported by an NSF Graduate Research Fellowship and a Fannie and John Hertz Foundation Fellowship} \\
\texttt{cliu568@mit.edu} \\
MIT
\and
Ankur Moitra\thanks{Supported by a grant from the ONR, NSF Award 1918656 and a David and Lucile Packard Fellowship.} \\
\texttt{moitra@mit.edu} \\
MIT
\and
Morris Yau \\
\texttt{morrisy@mit.edu} \\
MIT
}

\date{}
\begin{document}

\maketitle

\begin{abstract}
Recently Chen and Poor initiated the study of learning mixtures of linear dynamical systems. While linear dynamical systems already have wide-ranging applications in modeling time-series data, using mixture models can lead to a better fit or even a richer understanding of underlying subpopulations represented in the data. In this work we give a new approach to learning mixtures of linear dynamical systems that is based on tensor decompositions. As a result, our algorithm succeeds without strong separation conditions on the components, and can be used to compete with the Bayes optimal clustering of the trajectories. Moreover our algorithm works in the challenging partially-observed setting. Our starting point is the simple but powerful observation that the classic Ho-Kalman algorithm is a close relative of modern tensor decomposition methods for learning latent variable models. This gives us a playbook for how to extend it to work with more complicated generative models. 
\end{abstract}

\section{Introduction}

In this work, we study the problem of learning mixtures of linear dynamical systems from unlabelled trajectories. Given $k$ systems, each system evolves according to the following rules: for all $i \in [k]$, 
\begin{equation}
\label{eqn:lds-equations}
\begin{split}
    x_{t+1} &= A_i x_{t} + B_i u_t + w_t ,\\
    y_{t} &= C_i x_t + D_i u_t + z_t ,
\end{split}
\end{equation}
Here the $u_t$'s are control inputs to the system, the $w_t$'s are the process noise and the $z_t$'s are the observation noise. We observe the unlabelled input-output sequence $(u_1, y_1), (u_2, y_2), \cdots, (u_T, y_T)$
and the goal is to learn the underlying system parameters. When there is only one system, this is a classic problem in control theory called system identification \cite{aastrom1971system, ljung1998system}. A long line of recent works have established finite sample guarantees, often times from a single long trajectory, in increasingly more general settings \cite{hardt2018gradient, faradonbeh2018finite, hazan2018spectral, simchowitz2018learning, oymak2019non, tsiamis2019finite, sarkar2019nonparametric, simchowitz2019learning, bakshi2023new}. 

\emph{But what about mixture models?} Instead of one long trajectory, we observe many short trajectories. The main complication is that they are unlabelled \---- we don't know which system generated which trajectories. This problem has many potential applications. For example, the microbiome is a community of microorganisms that live in a host. They play a key role in human health and are affected by our environment in complex ways. In scientific studies, the composition of the microbiome is monitored over extended periods of time and researchers model its behavior using dynamical systems to discover new biological insights \cite{gonze2018microbial}. But when these dynamics are heterogenous across a population, it is natural to use a mixture model instead. More generally, there are wide-ranging applications of dynamical systems in biology and engineering and in many of these settings using a mixture model can lead to a better fit, or even a richer understanding of any underlying subpopulations represented in the data. 

However there is not much in the way of theoretical guarantees. In an important recent work, Chen and Poor gave the first efficient algorithms for learning mixtures of linear dynamical systems~\cite{chen2022learning}. They employed a two-stage approach where they use coarse estimates to cluster the trajectories and then, based on their clustering, further refine their estimates. Essentially, they use the stationary covariances to find subspaces according to which the trajectories from the systems are well-separated.  The setting in \cite{chen2022learning} is similar to ours but with two main distinctions. First in their model there are no inputs $u_t$ to the system.  Second they directly observe the sequence of states, whereas we get indirect observations. We elaborate on this below. Despite these differences, it is a useful comparison point especially in terms of assumptions and statistical guarantees.  

In this work, we give a new approach for learning mixtures of linear dynamical systems that is based on tensor decompositions. Our algorithm (Theorem~\ref{thm:main-thm}) achieves essentially optimal guarantees in many respects:

\begin{enumerate}
\item[(1)] Chen and Poor \cite{chen2022learning} require a number of strong and difficult to interpret technical conditions on the parameters. In contrast, we give an efficient algorithm for clustering that succeeds whenever clustering is possible (Theorem~\ref{thm:optimal-clustering}). In particular, whenever the systems have negligible statistical overlap as distributions, we will be able to find a clustering that misclassifies only a negligible fraction of the trajectories. 

\item[(2)] A priori it could be possible to learn the parameters of the mixture even when clustering is information-theoretically impossible. There is still useful information about the parameters that can be gleaned from the moments of the distribution. Indeed our algorithm succeeds under a condition we call \textit{joint non-degeneracy} (Definition~\ref{def:joint-nondeg}) which is a natural generalization of (individual) observability and controllability, both of which are standard assumptions in control theory and known to be necessary \cite{bakshi2023new}. These conditions hold even when the systems in the mixture model are almost entirely overlapping as distributions, rather than almost entirely disjoint and clusterable. Thus our algorithm brings results on learning mixtures of linear dynamical systems, which have complex time-varying behavior, in line with the strongest known guarantees for learning Gaussian mixture models \cite{kalai2010efficiently, belkin2010polynomial, moitra2010settling,bakshi2022robustly}. 

\item[(3)] Chen and Poor \cite{chen2022learning} work in the fully-observed setting, but without inputs, 
where we directly observe the sequence of states of the system
$x_1, x_2, \dots , x_T$.
In contrast, our algorithms work in the more challenging partially-observed setting where we only get indirect measurements $y_t$ of the hidden state. Even with just one system, this renders the maximum likelihood estimator a nonconvex optimization problem rather than a simpler linear regression problem. We also show that our algorithm succeeds with essentially optimally short trajectories. 

\end{enumerate}

Finally, our algorithms are based on a surprisingly undiscovered connection. The classic approach, going back to the 1960's, for solving system identification is to estimate the Markov parameters
$$[CB, CAB, CA^2B, \cdots, CA^{2s} B]$$
and use the Ho-Kalman algorithm (see Algorithm~\ref{algo:ho-kalman}) to recover estimates for $A,B$ and $C$ \cite{ho1966effective}. It turns out, the Ho-Kalman algorithm sets up a generalized eigenvalue problem, which just so happens to be the workhorse behind algorithms for low-rank tensor decompositions. In recent years, tensor methods have become a mainstay in theoretical machine learning, particularly for learning mixture models \cite{mossel2005learning, hsu2013learning, anandkumar2014tensor}. We leverage this connection along with modern tensor methods to teach the classic Ho-Kalman algorithm new tricks, namely we design a generalization of Ho-Kalman that can handle mixture models.

\section{Formal Setup and Assumptions}
\label{sec:formal-setup}

We begin by formally defining the problem.  Recall, a linear dynamical system $\calL$ follows the Markov process described in Equation \eqref{eqn:lds-equations},  
where each of  $A,B,C,D$ are matrices with dimensions  $n \times n$, $n \times p$, $m \times n$ and $m \times p$ respectively for some integers $m,n,p$.  The random variables $w_t$ and $z_t$ are typically modeled as standard normal corresponding to process and measurement noise.  Our observations are trajectories of the form $\{y_t, u_t\}_{t \in [\ell]}$.  In this paper, we will be interested in when our observations come from a heterogeneous mixture of LDS's as defined below.

\begin{definition}[Mixture of LDS's]
\label{definition:mixture-of-LDS}
 A mixture of linear dynamical systems is represented as $\calM = w_1 \calL_1 + \dots w_k \calL_k$, where $w_1 , \dots , w_k$ are positive real numbers summing to $1$ and $\calL_1((A_1,B_1,C_1,D_1)$, $ \dots , \calL_k(A_k, B_k,C_k,D_k)$ are each individual linear dynamical systems with the same dimensions (i.e. the same $m,n,p$).  The trajectories we observe are sampled according to the following process.  First an index $i \in [k]$ is drawn according to the mixing weights $w_1, \dots , w_k$ and then a trajectory of length $\ell$, denoted by $\{ (u_1, y_1), \ldots (u_\ell, y_{\ell}) \}$ is drawn from the corresponding dynamical system $\calL_i$. 
\end{definition}
Our input is a set of $N$ length $\ell$ trajectories denoted $\{ (u_1^{j}, y_1^{j}), \ldots, (u_\ell^{j}, y_\ell^{j}) \}$ for $j \in [N]$,  generated according to a mixture $\calM$ and our goal is to learn the parameters of the mixture, i.e. the individual linear dynamical systems and their mixing weights, given polynomially many such trajectories.   The question is whether we can learn the parameters when $\ell$ is small.  We will precisely specify the required trajectory length later on but essentially, the trajectory length we need is within a constant factor of optimal \--- in fact it is only a constant factor more than what is needed even when there is only a single component in the mixture.

It is important that our algorithm succeeds when the trajectory length $\ell$ is small because when $\ell$ is large enough, we could simply learn the parameters of each system from a single trajectory using \cite{bakshi2023new}.  In our main result, Theorem~\ref{thm:main-thm}, we give an algorithm that succeeds with essentially minimal trajectory length.  

For simplicity, throughout this paper, we will consider when all of the noise distributions (for all of the components of the mixture) are isotropic Gaussians i.e. $x_0 \sim \calD_0 = N(0, I_n), u_t \sim \calD_u = N(0, I_p) , w_t \sim \calD_w = N(0, I_n), z_t \sim \calD_z = N(0, I_m)$ although our results generalize to more general noise distributions as long as they have sufficiently many bounded moments.

\subsection{Assumptions for Learnability}
\label{subsec:assumptions}
Even for a single LDS, certain assumptions are necessary for learning to be possible.  We begin with standard definitions of the observability and controllability matrix \---- these govern whether a single LDS is learnable.  First, we define the \emph{observability matrix} which ensures that the interaction between the observation matrix $C$ and $A$ is not degenerate.     

\begin{definition}[Observability Matrix]\label{def:observability}
For an LDS $\calL(A,B,C,D)$ and an integer $s$, define the matrix $O_{\calL, s} \in \R^{sm \times n}$ as
\begin{equation*}
    O_{\calL, s} =\begin{bmatrix} C^\top   &   \Paren{CA}^\top &  \ldots  \Paren{ CA^{s-1} }^\top \end{bmatrix}^\top.
\end{equation*}
\end{definition}
An LDS is \emph{observable} if for some $s$, the matrix $O_s$ has full row rank. 
Similarly, we need to ensure that the interaction between the controller $B$ and $A$ is not degenerate. This is made precise by considering the \emph{controllability matrix}: 
\begin{definition}[Controllability Matrix]\label{def:controllability}
For an LDS $\calL(A,B,C,D)$ and  an integer $s$, define the matrix $Q_{\calL, s} \in \R^{n \times sp}$ as
\begin{equation*}
    Q_{\calL, s} = \begin{bmatrix} B & AB & \ldots & A^{s-1} B\end{bmatrix}
\end{equation*}
\end{definition}
An LDS is \emph{controllable} if for some $s$, the matrix $Q_s$ has full column rank. 
These two assumptions are necessary for the LDS to be learnable and in fact it is necessary to make a quantitatively robust assumption of this form (see \cite{bakshi2023new}).  In other words, we need a bound on the condition number of the observability and controllability matrices.

In addition to the assumptions required to learn a single linear dynamical system, we will require addition assumptions on the interaction of the LDS's to obtain learning algorithms for the mixture (as otherwise there could be degeneracies such as two components being almost the same which would make it information-theoretically impossible to learn).  

\subsubsection{Joint Non-degeneracy}

We introduce a  joint non-degeneracy condition that prevents certain degeneracies arising from the interaction between the components of the mixture e.g. if the components are too close to each other. 

\begin{definition}[Markov Parameters]
\label{def:markov-params}
Given a linear dynamical system, $\calL\Paren{ A,B, C,D}$, and an integer $T\geq 1$, the Markov Parameter matrix $G_{\calL, T}\in \R^{m \times (T + 1)p} $ is defined as the following block matrix:
\begin{equation*}
    G_{\calL, T} = \begin{bmatrix}
    D &  CB &  CAB &  \ldots & CA^{T-1}B
    \end{bmatrix}. 
\end{equation*}
\end{definition}

\begin{definition}[Joint Non-degeneracy]
\label{def:joint-nondeg}
For a mixture of LDS $\calM = w_1 \calL_1 + \dots + w_k \calL_k$ where each individual LDS is given by $\calL_i = \calL(A_i, B_i,C_i,D_i)$ (with the same dimension parameters $m,n,p$), we say $\calM$ is $(\gamma,s)$-jointly non-degenerate if for any real numbers $c_1, \dots , c_k$ with $ c_1^2 + \dots + c_k^2 = 1$, we have
\[
\norm{ c_1 G_{\calL_1, s} + \dots + c_k G_{\calL_k, s}}_F \geq \gamma \,.
\]
\end{definition}

The above condition ensures that no pair of components $\calL_i, \calL_j$ are too close to each other and enforces a stronger condition of robust linear independence of their Markov parameters.  However, it is satisfied for generic choices of the parameters e.g. when $mp \geq k$ we only need that $D_1, \dots , D_k$ are robustly linearly independent.  Note that this type of genericty assumption is very different from assuming some large separation between the components as in \cite{chen2022learning}.  We now state precisely the entire set of assumptions about the mixture $\calM$ that we require.

\begin{definition}[Well Behaved Mixture of LDS's]\label{def:well-behaved-mixture}
We say a mixture of LDS $\calM = w_1 \calL_1 + \dots + w_k \calL_k$ where each $\calL_i = \calL(A_i,B_i,C_i,D_i)$ is well-behaved if the following assumptions hold
\begin{itemize}
\item \textbf{Non-trivial Mixing Weights:} for some $w_{\min} > 0$, we have $w_i \geq w_{\min}$ for all $i \in [k]$.

\item \textbf{Non-trivial Individual Controllers and Measurements:} for all $i \in [k]$, $\norm{B_i}, \norm{C_i} \geq 1$
\item \textbf{Individual Boundedness:} for some parameter $\kappa$,
\[
\norm{A_i}, \norm{B_i}, \norm{C_i}, \norm{D_i} \leq \kappa \text{ for all } i \in [k].
\]
\item \textbf{Individual Observability and Controllability:}  for some integer $s$ and parameter $\kappa$, for all $i \in [k]$ the matrix $O_{\calL_i,s}$ has full row rank, the matrix $Q_{\calL_i, s}$ has full column rank and
 \begin{equation*}
        \begin{split}
            \sigma_{\max}(O_{2s})/\sigma_{\min}(O_s) &\leq \kappa, \\
            \sigma_{\max}(Q_{2s})/\sigma_{\min}(Q_s) &\leq \kappa.
        \end{split}
\end{equation*}
\item \textbf{Joint Non-degeneracy:} The mixture $\calM$ is $(\gamma,s)$ jointly non-degenerate for some parameter $\gamma > 0$.
\end{itemize}
\end{definition}

The assumptions on the individual components mirror those in \cite{bakshi2023new} where a more detailed discussion and justification can be found. 

\subsection{Our Results}

Now we can state our main theorem.

\begin{theorem}[Learning a Mixture of LDS's]
\label{thm:main-thm}
Given $0 < \epsilon, \delta < 1$, an integer $s$, and 
\[
N = \poly\Paren{m,n,p , s, \kappa, 1/w_{\min} ,  1/\gamma, 1/\eps, 1/\delta}
\]
observations $\Set{ (y_1^{i}, \ldots, y_{\ell}^{i} ) }_{i \in [N]}$, and the corresponding control inputs $\Set{ (u_1^{i}, \ldots, u_{\ell}^{i} ) }_{i \in [N]}$ of trajectory length $\ell \geq 6(s+1)$ , from a mixture of linear dynamical system $\calM = \sum_{i \in [k]} w_i \calL\Paren{ A_i,B_i,C_i,D_i}$, satisfying the assumptions in Section~\ref{subsec:assumptions}, there is an algorithm (Algorithm \ref{algo:learning-mixture}) that outputs estimates $\Set{ \hat{A}_i, \hat{B}_i, \hat{C}_i , \hat{D}_i }_{i \in [k]}$ such that with probability at least $1-\delta$, there is a permutation $\pi$ on $[k]$ such that for each $i \in [k]$, there exists a similarity transform $U_i$ satisfying 
\begin{equation*}
\begin{split}
 \max\Paren{ \Norm{A_{\pi(i)} - U^{-1}_i \hat{A}_i  U_i } , \Norm{C_{\pi(i)} - \hat{C}_iU_i } , \Norm{B_{\pi(i)} - U^{-1}_i \hat{B}_i } , \Norm{D_{\pi(i)} - \hat{D}_i},  \abs{ w_{\pi(i)} - \hat{w}_i }  } \leq \eps. 
\end{split}
\end{equation*}
Furthermore, Algorithm \ref{algo:learning-mixture} runs in $\poly\Paren{m,n,p , s, \kappa, 1/w_{\min} ,  1/\gamma, 1/\eps, 1/\delta}$ time.
\end{theorem}
\paragraph{Required Trajectory Length:}  Our algorithm requires trajectories of length $\ell = 6(s + 1)$ where $s$ is the observability/controllability parameter.  Note that trajectories of length $s$ are necessary as otherwise the parameters for even a single system are not uniquely recoverable so our required trajectory length is minimal up to this factor of $6$.



\section{Technical Overview}

For learning the parameters of a single LDS $\calL(A,B,C,D)$ from many observed trajectories, a standard recipe for this task is the algorithm of Ho-Kalman which succeeds at recovering $\hat{A}$ such that there exists a similarity transform $U$ satisfying $\Norm{A - U\hat{A}U^{-1}} = 0$ with analogous guarantees for $\hat{B}, \hat{C},$ and $\hat{D}$ in the infinite sample limit. Note that it is only possible to recover the system parameters under such an equivalence class of similarity transforms.  

The crux of the Ho-Kalman algorithm is to first estimate "Markov parameters" of the form $CA^iB$ for varying values of $i \in \Z^+$.  The Markov parameters are arranged in a corresponding Hankel matrix and an eigendecomposition style procedure is applied to the Hankel matrix to recover the system parameters (see Algorithm~\ref{algo:ho-kalman}).  However, the Ho-Kalman algorithm requires estimates of the Markov parameters which are difficult to obtain when the data is drawn from a mixture of LDS's instead of a single LDS.

Our general strategy is as follows.  For a particular Markov parameter $CA^{i}B$ we compute a carefully chosen $6$-th order tensor that can be estimated from the control inputs ($u_t$'s)  and observation ($y_t$'s). In particular, for a fixed $t$, given $N$ trajectories, we construct:
\begin{equation*}
\begin{split}
   & \wh{T}_i  =
    \frac{1}{N}\hspace{-0.04in}\sum_{j \in [N]} \hspace{-0.04in} y^j_{t+3i +2} \otimes u^j_{t+2i+2} \otimes y^j_{t+2i+1} \otimes u^j_{t+i+1} \otimes y^j_{t+i} \otimes u_i^{j}. 
\end{split}
\end{equation*}
We show that $\wh{T}_i$ is an unbiased estimator of a tensor whose components are the Markov parameters (see Lemma~\ref{lem:sixthmoment}):
\begin{equation}
\label{eqn:6th-order-markov-tensor}
    \wh{T}_i \sim \sum_{j \in [k]}w_j \Paren{ C_jA_j^{i}B_j} \otimes  \Paren{ C_jA_j^{i}B_j } \otimes  \Paren{ C_jA_j^{i}B_j}
\end{equation}
Brushing aside issues of sample complexity, we can assume we have access to the tensor in Eqn \eqref{eqn:6th-order-markov-tensor}. Ideally, we would just like to read off the components of this tensor and obtain the Markov parameters. 
However, provably recovering the components requires this tensor to be non-degenerate. To this end, we flatten the tensor along its first and second, third and fourth, and fifth and sixth modes to obtain a $3$-rd order tensor, whose components are the Markov paramters of the $j$-th LDS, flattened to a vector. In particular, we have
\begin{equation*}
    \tilde{T}_i = \sum_{j \in [k]} w_j v\Paren{ C_jA_j^{i}B_j}  \otimes  v\Paren{ C_jA_j^{i}B_j } \otimes  v\Paren{ C_jA_j^{i}B_j},
\end{equation*}
where $v\Paren{ C_jA_j^{i}B_j}$ simply flattens the matrix $ C_jA_j^{i}B_j$. 
The crux of our analysis is to show that the Joint Non-degeneracy condition (see Definition~\ref{def:joint-nondeg}) implies that components of the $3$-rd order tensor are (robustly) linearly independent (Lemma~\ref{lem:decompose-tensor}).  

Once we have established linear independence, we can run Jennrich's tensor decomposition algorithm (Algorithm~\ref{algo:jennrich}) on $\tilde{T}_i$ to obtain the components $w_j v\Paren{ C_jA_j^{i}B_j}  \otimes  v\Paren{ C_jA_j^{i}B_j } \otimes  v\Paren{ C_jA_j^{i}B_j}$. Assuming we know the mixing weights, we can just read off the first mode of this tensor, and construct the Markov parameter matrix. Once we have the Markov parameters, we can run (robust) Ho-Kalman (Algorithm~\ref{algo:ho-kalman}) to recover the $A_j, B_j, C_j$'s. 

However, in the setting where the mixing weights are unknown, we cannot hope simply read off the Markov parameter matrix from the component above. Instead, we can obtain the vectors $\tilde{v}_j=  w_j^{1/3} v\Paren{  C_jA_j^{i}B_j }$, for all $j \in [k]$, by simply reading the first mode and dividing out by the Frobenius norm of the second and third mode. We set up a regression problem where we solve for the coefficients $c_1, \ldots c_k$ as follows:
\begin{equation}
    \min_{c_1, c_2, \ldots , c_k} \Norm{ \sum_{l \in [k]} c_l \tilde{v}_l - \sum_{j\in[k]} w_j C_j A_j^{i} B_j  }_2^2,
\end{equation}
where we can estimate $\sum_{j\in[k]} w_j C_j A_j^{i} B_j $ up to arbitrary polynomial accuracy using the input samples. We show that the solution to this regression problem results in $c_l$'s that are non-negative and $c_l \sim w_{l}^{2/3}$ for all $l \in [k]$, which suffices to learn the mixing weights (see Theorem~\ref{thm:block-henkel-recovery} for details). We describe our complete algorithm in Section~\ref{sec:algorithm} and the analysis of each sub-routine in Section~\ref{sec:analysis}. 

\section{Related Work}

There is a long history of work on identifying/learning linear dynamical systems from measurements \cite{ding2013, zhang2011, spinelli2005, simchowitz2019learning, simchowitz2018, sarkar2019near, faradonbeh2017, shah2012linear, hardt2018gradient, hazan2018spectral, hazan2017}.  See \cite{Galrinho2016LeastSM} for a more extensive list of references.  These works focus on learning a the parameters of a linear dynamical system from a single long trajectory.   There has also been extensive empirical work on mixtures of time series and trajectories  which have been successfully applied in a variety of domains such as neuroscience, biology, economics, automobile design and many others \cite{bulteel2016clustering, mezer2009cluster, WongLi2000, kalliovirta2016gaussian, hallac2017toeplitz}.

Our setup can be viewed as a generalization of the more classical problem of learning mixtures of linear regressions which has been extensively studied theoretically \cite{CYCmixedregression13,CYSmixedregression13,LLmixedregression18, CLSmixedregression19, KHCmixedregression20,  DEF+mixedregression21}.  The fact that we receive many short trajectories parallels meta-learning framework in \cite{KSS+20mixedregression, KSKO20metamixedregression}.  However, the system dynamics in our setting (which are not present in standard mixed linear regression) make our problem significantly more challenging. It also has connections to super-resolution \cite{candes2014towards, moitra2015super, chen2021algorithmic} where tensor methods have also been employed \cite{huang2015super}. Finally, our model is similar to the well-studied switched linear dynamical system model (see~\cite{fox2008nonparametric,mudrik2022decomposed} and references therein).

The method of moments and tensor decomposition algorithms have been extensively used for learning a mixture of Gaussians, which can be viewed as the \textit{static} variant of learning a mixture of linear dynamical systems. A long line of work initiated by Dasgupta~\cite{dasgupta1999learning, arora2005learning, vempala2004spectral, achlioptas2005spectral, brubaker2008isotropic} gave efficient clustering algorithms for GMMs under various separation assumptions. 
Subsequently, efficient learning algorithms were obtained 
under information-theoretically minimal conditions~\cite{kalai2010efficiently, moitra2010settling, belkin2010polynomial, hardt2015tight}. \cite{hsu2013learning} used fourth-order tensor decomposition to obtain a polynomial-time algorithm for mixtures of spherical Gaussians with linearly independent means (with condition number guarantees). 
This result was extended via higher-order tensor decomposition for non-spherical Gaussian mixtures 
in a smoothed analysis model~\cite{ge2015learning}. Tensor decomposition has earlier been 
used in~\cite{frieze1996learning, goyal2014fourier} for the ICA problem. A combination of tensor decomposition and method-of-moments was also used to learn mixtures of Gaussians in the presence of outliers~\cite{bakshi2020outlier, bakshi2022robustly, liu2022learning, liu2023learning}.

\section{Moment Statistics of Linear Dynamical Systems}
\label{sec:moment-stats}

We begin with some basic properties of a single linear dynamical systems $\calL = \calL(A,B,C,D)$.  We explicitly compute certain moment statistics that will be used in our learning algorithm later on.


\begin{fact}[Algebraic Identities for LDS's, see \cite{bakshi2023new}]\label{fact:formula}
Let  $\calL\Paren{A,B,C,D}$ be a Linear Dynamical System. Then, for any $t \in \mathbb{N}$, 
\begin{equation*}
\begin{split}
    y_t & = \sum_{i = 1}^{t} \Paren{  C A^{i-1} B   u_{t-i}  + C A^{i-1}w_{t-i} } 
     + CA^{t}x_0 +  Du_t + z_{t} \,.
\end{split}
\end{equation*}
\end{fact}

\begin{fact}[Cross-Covariance of Control and Observation, see \cite{bakshi2023new}]\label{fact:markov-params-formula}
For any $t \in \mathbb{N}$, and any $j \geq 0$, given observations $y_t$ and control inputs $u_t$ from a linear dynamical system $\calL(A,B,C,D)$ that are independent with mean $0$ and covariance $I$, we have  $\expecf{}{y_{t+j} u_t^\top } = D$, if $j=0$, and $\expecf{}{y_{t+j} u_t^\top } = CA^{j-1}B$, otherwise.
\end{fact}
\begin{proof}
Invoking the algebraic identity from Fact~\ref{fact:formula}, consider the case where $j \neq 0$.
\begin{equation}
\label{eqn:expansion-2nd-moment}
\begin{split}
    \expecf{}{ y_{t+j} u_t^\top  } & = \expecf{}{  \Paren { \sum_{i = 1}^{t+j } \Paren{  C A^{i-1} B   u_{t+j-i}  + C A^{i-1}w_{t+j-i} } + CA^{t+j}x_0 +  Du_{t+j} + z_{t+j}  } u_t^\top } \\
    & = \underbrace{ \sum_{i = 1}^{j-1} CA^{i-1} B \expecf{}{u_{t+j-i} u_t^\top}   }_{\eqref{eqn:expansion-2nd-moment}.(1)}  + CA^{j-1} B \expecf{}{u_t u_t^\top }  + \underbrace{ \sum_{i = j+1}^{t+j} CA^{i-1} B \expecf{}{u_{t+j-i} u_t^\top} }_{\eqref{eqn:expansion-2nd-moment}.(2)}  \\
    & \hspace{0.2in} + \underbrace{ \sum_{i = 1}^{t+j } CA^{i-1} \expecf{}{w_{t+j-i} u_t^\top} + CA^{t+j} \expecf{}{ x_0 u_t^\top}  + D \expecf{}{ u_{t+j} u_t^\top } + \expecf{}{ z_{t+j }  u_t^\top} }_{ \eqref{eqn:expansion-2nd-moment}.(3) } \\
    & = CA^{j-1} B
\end{split}
\end{equation}
where the last inequality follows from observing that by independence of $u_t$'s,  $w_t$'s , $z_t$'s and $x_0$, the terms \eqref{eqn:expansion-2nd-moment}.(1), \eqref{eqn:expansion-2nd-moment}.(2) and \eqref{eqn:expansion-2nd-moment}.(3) are $0$. Similarly, when $j=0$, the only non-zero term is $\expecf{}{ D u_t u_t^\top } = D$ and the claim follows.
\end{proof}

In light of the above, we make the following definition. 
\begin{definition}[System Parameters]\label{def:estimator-means}
 For an LDS $\calL(A,B,C,D)$ and an integer $j \geq 0$, we define the matrix $X_{\calL, j} = D $ if $j=0$ and $X_{\calL, j} = CA^{j-1}B $ if $j>0$. 
\end{definition}
Next, we construct a sixth moment tensor that, restricted to a single LDS, gives an unbiased estimator for a tensor of the system parameters.


\begin{lemma}[Sixth-moment Statistics]\label{lem:sixthmoment}
Given a linear dynamical system $\calL(A, B , C, D)$ and integers $t, k_1, k_2, k_3 \geq 0$, let $t_1= t+k_1$, $t_2 = t_1 +k_2$ and $t_3 = t_2 + k_3$. Then, we have
\begin{equation*}
\begin{split}
&\expecf{}{y_{t_3 + 2} \otimes u_{t_2 + 2} \otimes y_{t_2+ 1}   \otimes u_{t_1 + 1} \otimes  y_{t_1} \otimes u_{t}  }  = X_{\calL, k_3} \otimes X_{\calL, k_2} \otimes  X_{\calL, k_1} \,, 
\end{split}
\end{equation*}
where $X_{\calL,j}$ is defined in Definition~\ref{def:estimator-means}. 
\end{lemma}
\begin{proof}
First, for simplicity consider the case where $k_1=  k_2 = k_3 =0$. Then, 

\begin{equation}
\begin{split}
    \expecf{}{ y_{t + 2}  \otimes u_{t+ 2} \otimes y_{t+ 1} \otimes u_{t+1} \otimes y_{t} \otimes u_t } &  = \expecf{}{ D u_{t+2} \otimes u_{t+2} \otimes D u_{t+1} \otimes u_{t+1} \otimes D u_t \otimes u_t } \\
    & = \Paren{ D \expecf{}{u_{t+2 } u_{t+2}^\top } }  \otimes \Paren{ D \expecf{}{u_{t+1 } u_{t+1}^\top } }  \otimes \Paren{ D \expecf{}{u_{t } u_{t}^\top } } \\
    & = D \otimes D \otimes D,
\end{split}
\end{equation}
where the second equality follows from $u_{t+2}, u_{t+1}$ and $u_{t}$ being independent random variables. 
Next, consider the case where $k_3=0$ and $k_2, k_1 >0$. Observe, $y_{t+k_1 + k_2 + 2} \otimes u_{t+ k_1 + k_2 + 2}$ has only one non-zero term in expectation. Therefore, we can split the sum as follows: 

\begin{equation}
\label{eqn:k-3=0}
    \begin{split}
    & \expecf{}{ y_{t+ k_1 + k_2 +2  } \otimes u_{t+k_1 +k_2 + 2}  \otimes y_{t+k_1 + k_2 + 1} \otimes u_{t+k_1 + 1}  \otimes y_{t+k_1} \otimes u_{t}  } \\
    & = \expecf{}{ D \Paren{ u_{t+ k_1 + k_2 + 2 }\otimes u_{t+ k_1 + k_2 + 2 } }  \otimes \Paren{ y_{t+k_1 + k_2 + 1}  \otimes  u_{t+k_1 + 1}   } \otimes \Paren{y_{t+k_1} \otimes u_{t }}  }  \\
    & = D \Paren{ \expecf{}{u_{t+ k_1 + k_2 + 2 }\otimes u_{t+ k_1 + k_2 + 2 }}  } \otimes \underbrace{ \expecf{}{ \Paren{ y_{t+k_1 + k_2 + 1}  \otimes  u_{t+k_1 + 1}    } \otimes \Paren{y_{t+k_1} \otimes u_{t}}  } }_{\eqref{eqn:k-3=0}.(1) }  \\
    \end{split}
\end{equation}
where the second equality follows from observing that $u_{t+k_1 +k_2 + 2}$ is independent of all the terms appearing in the expansion of $y_{t+k_1+k_2 +1}$  and $y_{t+k_1}$, and the random variables $u_{t+ k_1 + k_2 + 2}$ and $u_{t+k_1}$. 

Now, we focus on simplifying term \eqref{eqn:k-3=0}.(1). Let $\zeta_t = CA^t x_0 + Du_t + z_t$. Observe that $\expecf{}{ w_t \otimes u_{t'} } = \expecf{}{ CA^t x_0 \otimes u_t }= \expecf{}{ z_t \otimes u_t }  = 0$,  for all $t,t'$, and $\expecf{}{ u_{t_1} \otimes u_{t_2} \otimes u_{t_3} \otimes u_{t_4}  } = 0$ for all $t_1> t_2 > t_3 > t_4$. Further, any permutation of $t_1, t_2, t_3$ and $t_4$ is also $0$.  Plugging in the definition from Fact \ref{fact:formula},  we have

\begin{equation}
\label{eqn:4-th-tensor}
    \begin{split}
        & \EE \Big[  y_{t+k_1 + k_2 + 1} \otimes u_{t + k_1 + 1} \otimes y_{t + k_1} \otimes u_t \Big]  \\
        & = \EE \Bigg[ \Paren{ \sum_{i = 1}^{t+k_1 +k_2 +1 } \Paren{  C A^{i-1} B   u_{t+k_1 +k_2 +1 -i}  + C A^{i-1}w_{t+k_1 +k_2 +1 -i} } + \zeta_{t+k_1 +k_2 +1 } } \otimes u_{t+k_1 + 1}  \\
        & \hspace{0.6 in} \otimes \Paren{\sum_{i = 1}^{t+k_1  } \Paren{  C A^{i-1} B   u_{t+k_1   -i}  + C A^{i-1}w_{t+k_1   -i} } + \zeta_{t+k_1   }   } \otimes  u_t \Bigg] \\
        &  = \EE \Bigg[ \Paren{ \sum_{i = k_2}^{t+k_1 +k_2 +1 } \Paren{  C A^{i-1} B   u_{t+k_1 +k_2 +1 -i}   }  } \otimes u_{t+k_1 + 1}  \\
        & \hspace{0.6  in} \otimes \Paren{\sum_{i = 1}^{t+k_1  } \Paren{  C A^{i-1} B   u_{t+k_1   -i}  + C A^{i-1}w_{t+k_1   -i} }    } \otimes  u_t \Bigg] \\
        & = \underbrace{ \EE \Bigg[ \Paren{  C A^{k_2-1} B       u_{t+k_1 +1 }  }  \otimes u_{t+k_1 + 1} \Bigg] }_{ \eqref{eqn:4-th-tensor}.(1) }  \otimes \underbrace{ \EE \Bigg[  \Paren{\sum_{i = 1}^{t+k_1  } \Paren{  C A^{i-1} B   u_{t+k_1   -i}  + C A^{i-1}w_{t+k_1   -i} }    } \otimes  u_t \Bigg]  }_{\eqref{eqn:4-th-tensor}.(2) } \\
        & \hspace{0.2 in}  + \underbrace{ \EE \Bigg[ \Paren{ \sum_{i = k_2+1}^{t+k_1 +k_2 +1 } \Paren{  C A^{i-1} B   u_{t+k_1 +k_2 +1 -i}   }  } \otimes u_{t+k_1 + 1}  \otimes    y_{t+k_1} \otimes  u_t \Bigg] }_{\eqref{eqn:4-th-tensor}.(3)}   
    \end{split}
\end{equation} 
where the second equality follows from observing that $\expecf{}{ w_{t+k_1 k_2 +1 -i } \otimes u_{t+ k_1 + 1} \otimes w_{t+ k_1 -j} \otimes u_t } = 0 $ for all $i \in [1, t+ k_1 + k_2 + 1 ]$ and $j \in [1, t+ k_1]$. Similarly, $\expecf{}{\zeta_{t+ k_1 + k_2 + 1} \otimes u_{t+k_1 + 1} \otimes  
y_{t+k_1}\otimes u_t } =0$. Further, for all $i \in [1, k_2 -1]$, $\expecf{}{ u_{t+k_1 + k_2 + 1 -i} \otimes u_{t+k_1 + 1} \otimes y_t \otimes u_t } =0$. The third equality follows from observing that $u_{t+k_1 + 1}$ is independent of $y_t \otimes u_t$.
Next, observe 
\begin{equation}
\label{eqn:4th.1}
    \eqref{eqn:4-th-tensor}.(1) =  CA^{k_2 - 1} B, 
\end{equation}
since $\expecf{}{u_{t+k_1 +1 } \otimes u_{t+k_1 +1 } } = I$. Using a similar argument, we observe that all the terms in \eqref{eqn:4-th-tensor}.(2) are zero in expectation apart from the one corresponding to $CA^{k_1-1}B$. Therefore, 
\begin{equation}
\label{eqn:4th.2}
\eqref{eqn:4-th-tensor}.(2) = CA^{k_1 - 1} B.
\end{equation}
Next, recall that $\expecf{}{u_t}=0$ for all $t$, and since $u_{t+k_1 + 1}$ is independent of all $u_{t'}$ where $t'<  t+ k_1+1$, 
\begin{equation}
\label{eqn:4th.3}
    \begin{split}
        \eqref{eqn:4-th-tensor}.(3) & =  0.
    \end{split}
\end{equation}
Similarly, when $k_1=0$,  $\eqref{eqn:4-th-tensor}.(1)= D$ and when $k_2=0$, $\eqref{eqn:4-th-tensor}.(2)= D$. 

Therefore, combining equations \eqref{eqn:4th.1},\eqref{eqn:4th.2} and \eqref{eqn:4th.3}, and plugging them back into equation \eqref{eqn:k-3=0}, we have 
\begin{equation}
   \begin{split}
     \expecf{}{ y_{t+ k_1 + k_2 +2  } \otimes u_{t+k_1 +k_2 + 2}  \otimes y_{t+k_1 + k_2 + 1} \otimes u_{t+k_1 + 1}  \otimes y_{t+k_1} \otimes u_{t}  } 
    & = D   \otimes  X_{\calL , k_2} \otimes X_{\calL,k_1}  \\
    \end{split} 
\end{equation}

It remains to consider the case where $k_3 >0$. We can now simply repeat the above argument and observe that instead of picking up the term $D u_{t+k_1 + k_2 + k_3 + 2  }$ from the expansion of $y_{t+k_1 + k_2 +k_3 +2}$, we now pick up the term $CA^{k_3-1} B u_{t+k_1 + k_2 + k_3 + 2  }$. This concludes the proof.
\end{proof}


Now consider a mixture of LDS $\calM = w_1\calL_1 + \dots  + w_k\calL_k $ where $\calL_i = \calL(A_i, B_i, C_i, D_i)$.  Using Lemma~\ref{lem:sixthmoment}, we have an expression for the sixth moments of the mixture.

\begin{corollary}\label{coro:mixture-fourthmoment}
For a mixture of LDS $\calM = w_1\calL_1 + \dots  + w_k\calL_k $ and for $t, k_1, k_2, k_3 \geq 0$, let $t_1= t+k_1$, $t_2 = t_1 +k_2$ and $t_3 = t_2 + k_3$. Then, 
\begin{equation*}
\begin{split}
&\expecf{\calM}{ y_{t_3 + 2} \otimes u_{t_2 + 2} \otimes y_{t_2 + 1}\otimes u_{t_1 + 1} \otimes  y_{t_1 } \otimes u_{t}}   = \sum_{i = 1}^k w_i X_{\calL_i, k_3} \otimes X_{\calL_i, k_2} \otimes  X_{\calL_i, k_1}.
\end{split}
\end{equation*}
\end{corollary}
\begin{proof}
This follows from linearity of expectation combined with Lemma~\ref{lem:sixthmoment}. 
\end{proof}

\subsection{Singular Value Bounds for Observability and Controllability Matrices}
Later on, we will also need the following two claims from \cite{bakshi2023new} that give us bounds on the singular values of $O_{\calL_i, s}, Q_{\calL_i,s}, A^s$. 
\begin{claim} [Claim 5.15 in \cite{bakshi2023new}]\label{claim:eigenvalue-bound}
Consider a well-behaved mixture of LDS (Definition~\ref{def:well-behaved-mixture}) $\calM = w_1 \calL_1 + \dots + w_k \calL_k$ where each $\calL_i = \calL(A_i,B_i,C_i,D_i)$. Then for all $i \in [k]$, $\sigma_{\min}(O_{\calL_i,s}) \leq \sqrt{s}\kappa$ and  $\sigma_{\min}(Q_{\calL_i,s}) \leq \sqrt{s}\kappa$.
\end{claim}

\begin{claim}[Claim 5.16 in \cite{bakshi2023new}]\label{claim:boundA} Consider a well-behaved mixture of LDS (Definition~\ref{def:well-behaved-mixture}) $\calM = w_1 \calL_1 + \dots + w_k \calL_k$ where each $\calL_i = \calL(A_i,B_i,C_i,D_i)$.  Then for any integer $t > 0$, and for all $i \in [k]$, 
\[
\norm{A_i^t}_F \leq (\sqrt{n} \kappa)^{t/s} \,. 
\]
\end{claim}

\section{Algorithm}
\label{sec:algorithm}

In this section, we describe our algorithm for learning a mixture of Linear Dynamical Systems. At a high level, our algorithm uses multiple trajectories to obtain an estimate of the tensor 
\begin{equation*}
\Pi_{\calM} = \sum_{i \in [k]}  w_i G_{\calL_i, 2s} \otimes G_{\calL_i, 2s} \otimes G_{\calL_i, 2s} \,.
\end{equation*}
where 
\begin{equation*}
G_{\calL_i, s} = \begin{bmatrix}
    D_i &  C_i B_i &  C_i A_i B_i &  \ldots & C_i A_i^{s-1}B_i
    \end{bmatrix}.
\end{equation*}

Recall that $G_{\calL_i, 2s}$ has blocks that are of the form $X_{\calL_i, s'}$ for $s' \leq 2s$ and thus it follows from Corollary~\ref{coro:mixture-fourthmoment} that we can construct a unbiased estimates of the individual blocks
\[
T_{k_1, k_2,k_3} = \sum_{i \in [k]}  w_i X_{\calL_i, k_1} \otimes X_{\calL_i, k_2} \otimes X_{\calL_i, k_3}
\]
of this tensor from the observations and control input.  Piecing together the individual blocks lets us construct an estimate of  $\Pi_{\calM}$. Since we have access to multiple \textit{independent} trajectories, we can show that the variance is bounded and we indeed have access to a tensor close to $\Pi_{\calM}$. 

Essentially, we then run the classical Jennrich's tensor decomposition algorithm on the tensor $\Pi_{\calM}$ to recover the factors $G_{\calL_i, 2s}$. The key is that the joint nondegeneracy assumption implies that vectors obtained by flattening $G_{\calL_1, 2s}, \dots , G_{\calL_k, 2s }$ are  (robustly) linear independent. Therefore, Jennrich's algorithm recovers the factors $G_{\calL_1, 2s}, \dots , G_{\calL_k, 2s }$.  These are exactly the Markov parameters of the individual components and we can then invoke a robust variant of Ho-Kalman~\cite{oymak2019non} to recover the corresponding parameters.

\begin{mdframed}
  \begin{algorithm}[Learning a Mixture of LDS's]
    \label{algo:learning-mixture}\mbox{}
    \begin{description}
\item[Input:] $N$ sample trajectories of length $\ell$ from a mixture of LDS $\calM = \sum_{i \in [k]} w_i \calL\Paren{A_i, B_i, C_i, D_i}$ denoted $\Set{(y_1^i, \dots , y_\ell^i)}_{i \in [N]}$, the corresponding control inputs $\Set{(u_1^i, \dots , u_\ell^i) }_{i \in [N]}$,  parameter $s \in \N$ for individual observability and controllability and joint nondegeneracy, accuracy parameter $0< \eps<1$ and allowable failure probability $0<\delta<1$. 
     \item[Operation:]\mbox{}
    \begin{enumerate}
    \item Run Algorithm \ref{algo:learning-subspaces} on the input samples  and let $\Set{\wt{G}_i }_{i \in [k]}$ be the matrices returned  
    \item For $0 \leq k_1 \leq 2s$, compute estimate $\wh{R}_{k_1}$ of $\EE_{\calM}\left[ y_{k_1 + 1} \otimes u_1\right]$ as
    \[
    \wh{R}_{k_1} = \frac{1}{N} \sum_{i = 1}^N y^i_{k_1 + 1} \otimes u_1^i \,.
    \]
    \item Construct estimate $\wh{R}_{\calM}$ of $R_{\calM}$ by stacking together estimates $\wh{R}_{0}, \wh{R}_{1}, \dots , \wh{R}_{2s - 1}$ 
    \item Solve for weights $\wt{w_1}, \dots , \wt{w_k}$ that minimize
    \[
    \norm{\wt{w_1}\wt{G}_1 + \dots + \wt{w_k}\wt{G}_k - \wh{R}_{\calM} }_F 
    \]
    \item Set $\hat{G}_i = \wt{G}_i/\sqrt{\wt{w_i}}$
    \item Set $\hat{w}_i = \wt{w_i}^{3/2}$ for all $i \in [k]$
    \item Run Algorithm \ref{algo:ho-kalman} on  $\hat{G}_i$ for each $i \in [k]$ to recover parameters $\Set{ \hat{A}_i, \hat{B}_i, \hat{C}_i, \hat{D}_i }_{i \in [k]}$  
    \end{enumerate}
    \item[Output:]  The set of parameter estimates $\Set{\hat{w}_i, \hat{A}_i, \hat{B}_i, \hat{C}_i, \hat{D}_i }_{i \in [k]}$ 
    \end{description}
  \end{algorithm}
\end{mdframed}

However, there is a minor technical complication in the above approach due to the unknown mixing weights.  Tensor decomposition cannot directly recover $G_{\calL_1, 2s}$ due to the unknown mixing weights but can only recover a constant multiple of it.  To resolve this issue, we note that we can recover estimates of the rank-$1$ factors  $w_i G_{\calL_i, 2s} \otimes G_{\calL_i, 2s} \otimes G_{\calL_i, 2s}$ and from this, we can construct estimates $\wt{G}_{i}$ for the vectors $w_i^{1/3} G_{\calL_i, 2s}$ by ``taking the cube root" (see Algorithm~\ref{algo:learning-subspaces}).  We can then solve a separate regression problem using only second moments to recover the mixing weights and deal with the scaling issue (see steps $2-6$ in Algorithm~\ref{algo:learning-mixture}).  It will be useful to define
\[
R_{\calM} = \sum_{i \in [k]} w_i G_{\calL_i, 2s},
\]
which we can also estimate empirically by estimating each block
\[
R_{k_1} = \EE_{\calM}[y_{k_1 + 1} \otimes u_1] = \sum_{i \in [k]} w_i X_{\calL_i, k_1},
\]
separately  (recall Fact~\ref{fact:markov-params-formula}).

\begin{mdframed}
  \begin{algorithm}[Learn Individual Markov Parameters]
    \label{algo:learning-subspaces}\mbox{}
    \begin{description}
\item[Input:] $N$ sample trajectories of length $\ell$ from a mixture of LDS $\calM = \sum_{i \in [k]} w_i \calL\Paren{A_i, B_i, C_i, D_i}$, denoted by $\Set{ (y_1^{i}, \ldots y^{i}_{\ell})}_{i \in [N]}$, the corresponding control inputs, $\Set{ (u^{i}_1, \ldots, u^{i}_{\ell})}_{i \in [N]}$,  parameter $s \in \N$ for individual observability and controllability and joint nondegeneracy, accuracy parameter $\eps$ and allowable failure probability $\delta$.
    \item[Operation:]\mbox{}
    \begin{enumerate}
    \item For $0 \leq k_1\leq 2s,0 \leq  k_2 \leq 2s , 0 \leq k_3 \leq 2s$,  
    \begin{enumerate}
    \item Compute empirical estimate $\wh{T}_{k_1,k_2,k_3}$    as follows:
    \[
    \begin{split}
        \wh{T}_{k_1, k_2, k_3} = \frac{1}{N} \sum_{ i \in [n] } y^{i}_{k_1 +k_2 + k_3 + 3 } \otimes u^{i}_{k_1 +k_2 + 3}  \otimes y^{i}_{k_1 +k_2 + 2 }\otimes u^{i}_{k_1 + 2} \otimes y^{i}_{k_1 + 1} \otimes u^{i}_1
    \end{split}
    \]

    \end{enumerate}

    \item Construct estimate $\wh{\Pi}_{\calM}$ for $\Pi_{\calM}$ by piecing together the blocks $\wh{T}_{k_1,k_2,k_3}$ into a $(2s + 1) \times (2s + 1) \times (2s+1)$ grid
    \item Flatten pairs of dimensions of $\wh{\Pi}_{\calM}$ so that it is a order-$3$ tensor with dimensions $(2s + 1)mp \times (2s + 1)mp \times (2s+1)mp$ 
    \item Run Jennrich's algorithm (Algorithm \ref{algo:jennrich}) to obtain the following decomposition
    \[
    \wh{\Pi}_{\calM} = \wh{T}_1 + \dots + \wh{T}_k
    \]
    \item For each $\wh{T}_i$, compute the Frobenius norm of each slice in its second and third dimensions to obtain a vector $\wh{v}_i \in \R^{(2s + 1)mp}$ 
    \item Construct $\wt{G}_i$ by rearranging the vector $\wh{v}_i/\norm{\wh{v}_i}^{2/3}$ back into an $m \times (2s+1)p$ matrix (undoing the flattening operation) 
    \end{enumerate}
    \item[Output:] Matrices $\wt{G}_1, \dots , \wt{G}_k$
    \end{description}
  \end{algorithm}
\end{mdframed}

\begin{mdframed}
  \begin{algorithm}[Parameter Recovery via Ho-Kalman~\cite{oymak2019non}]
    \label{algo:ho-kalman}\mbox{}
    \begin{description}
\item[Input:] Parameter $s$, Markov parameter matrix estimate $\hat{G} = [\hat{X}_0, \dots , \hat{X}_{2s}] $  
    
    \item[Operation:]
    \begin{enumerate}
    \item Set $\hat{D} = \hat{X}_0$
    \item Form the Hankel matrix $\hat{H} \in \R^{ms \times p(s+1)}$ from $\hat{G}$ as
    \[
    \hat{H} = \begin{bmatrix} \hat{X}_1 & \hat{X}_2 & \dots & \hat{X}_{s+1} \\ \hat{X}_2 & \hat{X}_3 & \dots & \hat{X}_{s + 2} \\ \vdots & \vdots & \ddots & \vdots \\ \hat{X}_{s} & \hat{X}_{s+1} & \dots & \hat{X}_{2s} 
    \end{bmatrix}
    \]
    \item Let $\hat{H}^{-} \in \R^{ms \times ps}$  be the first $ps$ columns of $\hat{H}$.
    \item Let $\hat{L} \in \R^{ms \times ps}$ be a rank $n$ approximation of $\hat{H}^{-}$ obtained by computing the SVD of $\hat{H}^{-}$ and truncating all but the largest $n$ non-zero singular values. Compute $U \Sigma V = SVD(\hat{L})$. 
    \item Set $\hat{O} \in \R^{ms \times n}$ to be  $U\Sigma^{1/2}$ and $\hat{Q} \in \R^{n \times ps}$ to be $\Sigma^{1/2}V^\top$ .
    \item Set $\hat{C}$ to be the first $m$ rows of $\hat{O}$. Set $\hat{B}  $ to be the first $p$ column of $\hat{Q}$.
    \item Set $\hat{H}^{+} \in \R^{ms \times ps} $  to be the last $ps$ column of $\hat{H}$.
    \item Compute $\hat{A} = \hat{O}^{\dagger} \hat{H}^{+} \hat{Q}^{\dagger}$
    \end{enumerate}
    \item[Output:] $\hat{A} \in \R^{n \times n}, \hat{B} \in \R^{n \times p}, \hat{C} \in \R^{m \times n}, \hat{D} \in \R^{m \times p}$  
    \end{description}
  \end{algorithm}
\end{mdframed}

\section{Analysis}
\label{sec:analysis}

In this section, we provide the analysis of the algorithms we presented in Section~\ref{sec:algorithm}.  At the end, we will prove our main theorem, Theorem~\ref{thm:main-thm}.  Throughout the remainder of this section, we will assume $\ell = 6(s + 1)$ since otherwise, we can simply truncate all observed trajectories to length exactly $6(s + 1)$.

\subsection{Recovering the Markov Parameters}

We proceed by analyzing each sub-routine separately. In particular, Algorithm~\ref{algo:learning-mixture} proceeds by first taking the input samples and running Algorithm~\ref{algo:learning-subspaces} to learn the individual sets of Markov parameters up to some scaling by the mixing weights. Formally, 

\begin{theorem}[Recovering the Markov Parameters]
\label{thm:span-recovery}
Given $\eps, \delta>0$ and  
\[N \geq \poly(m, n, p, s, \kappa, 1/w_{\min}, 1/\gamma, 1/\eps,  1/\delta )
\]
trajectories from a mixture of LDS's, $\calM = \sum_{i \in[k]} w_i \calL\Paren{A_i, B_i, C_i, D_i}$,  Algorithm~\ref{algo:learning-subspaces} outputs a set of matrices $\wt{G}_1, \wt{G}_2, \ldots \wt{G}_k$ such that with probability $1 - \delta$, there is a permutation $\pi$ on $[k]$ such that 
\[
\norm{\wt{G}_{\pi(i)} - w_i^{1/3}G_{\calL_i, 2s}}_F \leq \eps ,
\]
for all $i \in [k]$.  Further, Algorithm~\ref{algo:learning-subspaces} runs in $\poly(N)$ time.
\end{theorem}

We first show that the empirical  $6$-th moment tensor is close to the true tensor in Frobenius norm.

\begin{lemma}[Empirical Concentration of the 6-th Moment]\label{lem:empirical-estimate-tensor}
Given $\eps, \delta>0$ and $N\geq N_0$ length $6s$ trajectories from a mixture of linear dynamical systems $\calM = \sum_{i\in [k]} w_i \calL(A_i, B_i, C_i, D_i)$, if 
\[
N_0 \geq \poly(m,n,p,s ,  \kappa, 1/w_{\min}, 1/\gamma, 1/\eps, 1/\delta ),
\]
with probability at least $1-\delta$
Algorithm~\ref{algo:learning-subspaces} computes a tensor $\wh{\Pi}_{\calM}$ such that
\[
\Norm{\wh{\Pi}_{\calM}  - \Pi_{\calM}}_F \leq \eps ,
\]
where $\Pi_{\calM} = \sum_{i \in [k]}  w_i G_{\calL_i, 2s} \otimes G_{\calL_i, 2s} \otimes G_{\calL_i, 2s}$.
\end{lemma}
\begin{proof}
Note that the joint distribution of $(u_1^i, \dots , u_\ell^i, y_1^i , \dots , y_\ell^i)$ is Gaussian.  Furthermore, since $\ell = 6(s + 1)$, by Claim~\ref{claim:boundA}, the formula in Fact~\ref{fact:formula} and the assumptions in Definition~\ref{def:well-behaved-mixture}, the covariance of this Gaussian has entries bounded by $\poly(m,n,p,s,\kappa)$.  Thus, by standard concentration inequalities, the empirical sixth moment tensor concentrates around its mean with high probability.  Since $\Pi_{\calM}, \wh{\Pi}_{\calM}$ are obtained by taking a linear transformation of the sixth moment tensor and the coefficients of this transformation are also bounded by $\poly(m,n,p,s,\kappa)$, we are done.

\end{proof}

Next, we show that running Jennrich's algorithm on an appropriate flattening of the tensor $\wh{\Pi}_{\calM}$ recovers an estimate of the Markov parameters of each individual component of the mixture.

\begin{lemma}[Markov Parameters via Tensor Decomposition]\label{lem:decompose-tensor}
Given $\eps, \delta>0$ and $N\geq N_0$ length $6s$ trajectories from a mixture of linear dynamical systems $\calM = \sum_{i\in [k]} w_i \calL(A_i, B_i, C_i, D_i)$, if 
\[
N_0 \geq \poly(m,n,p,s ,  \kappa, 1/w_{\min}, 1/\gamma, 1/\eps, 1/\delta ),
\]
with probability at least $1-\delta$, Jennrich's algorithm (Algorithm~\ref{algo:jennrich}) outputs tensors $\wh{T}_1, \wh{T}_2, \ldots , \wh{T}_k$ such that there is some permutation $\pi$ on $[k]$ such that for all $i \in [k]$,
\[
\Norm{\wh{T}_{\pi(i)} - w_i \cdot  v(G_{\calL_i, 2s}) \otimes v(G_{\calL_i, 2s}) \otimes v(G_{\calL_i, 2s})}_F \leq \eps, 
\]
where  $v(G_{\calL_i, 2s})$ denotes flattening the matrix $G_{\calL_i, 2s}$ into a $mp(2s+1)$-dimensional vector.
\end{lemma}
\begin{proof}
Let $K$ be the matrix whose columns are $v(G_{\calL_i, 2s})$.  By the joint nondegeneracy assumption, $\sigma_k(K) \geq \gamma$ .  On the other hand by Claim~\ref{claim:boundA} and the assumptions about the individual components of the mixture, we know that 
\[
\norm{K}_F \leq \poly(k,m,p,n,s, \kappa)
\]
so we can apply Theorem~\ref{thm:robust-jennrich} and Lemma~\ref{lem:empirical-estimate-tensor} (with $\eps$ rescaled appropriately by a polynomial in the other parameters) to get the desired bound.
\end{proof}

Now we can complete the proof of Theorem~\ref{thm:span-recovery}.
\begin{proof}[Proof of Theorem~\ref{thm:span-recovery}]
Lemma~\ref{lem:decompose-tensor} implies that  
\[
\norm{\wh{v}_i - w_i \cdot v(G_{\calL_i, 2s}) \norm{v(G_{\calL_i, 2s} }^2 } \leq \eps \,.
\]
This also implies that 
\[
\left \lvert \norm{\wh{v}_i} - w_i \norm{v(G_{\calL_i, 2s} }^3 \right \rvert \leq \eps \,.
\]
Also note that we must have
\[
1 \leq \norm{v(G_{\calL_i, 2s})} \leq \poly(k,m,n,p,s, \kappa) \,.
\]
Thus
\[
\norm{\wh{v}_i/\norm{\wh{v}_i}^{2/3} - w_i^{1/3} v(G_{\calL_i, 2s})  } \leq \eps \cdot \poly(k,m,n,p,s, \kappa) \,.
\]
We now get the desired bound by simply rescaling the setting of $\eps$ in Lemma~\ref{lem:decompose-tensor} by a polynomial in the other parameters.
\end{proof}

\subsection{Recovering the Mixing Weights}

Next, we argue about the mixing weights $\wt{w}_1, \dots , \wt{w}_k$ computed in the regression step in Algorithm~\ref{algo:learning-mixture}.

\begin{theorem}[Recovering the Mixing Weights]
\label{thm:block-henkel-recovery}
Assume that the matrices $\wt{G}_i$ computed in Algorithm~\ref{algo:learning-mixture} satisfy Theorem~\ref{thm:span-recovery}.  Then, with $1 - \delta$ probability, the mixing weights $\wt{w}_1, \dots , \wt{w}_k$ computed in Algorithm~\ref{algo:learning-mixture} satisfy 
\[
|\wt{w}_{\pi(i)} - w_i^{2/3}| \leq \eps \cdot \poly(\kappa,m,n,s,p, 1/\gamma , 1/w_{\min})
\]
for all $i \in [k]$.
\end{theorem}
\begin{proof}
Recall by Fact~\ref{fact:markov-params-formula} and Definition~\ref{def:estimator-means} that $\EE[\wh{R}_{\calM}] = R_{\calM}$.  Also by the same argument as in the proof of Lemma~\ref{lem:empirical-estimate-tensor}, the empirical estimate concentrates with high probability since the observations and control inputs are jointly Gaussian with bounded covariance.  Thus, with $1 - \delta$ probability, we have $\norm{ R_{\calM} - \wh{R}_{\calM}}_F \leq \eps$.  Recalling the definition of $R_{\calM}$ and applying Theorem~\ref{thm:span-recovery}, we must have that
\[
\norm{ \wh{R}_{\calM} - w_i^{2/3}\wt{G}_{\pi(i)}}_F \leq \eps (k + 1) \,.  
\]
Now consider any other set of choices for $\wt{w}_{\pi(i)}$.  We must have that 
\[
\norm{ \sum_{i = 1}^k(w_i^{2/3} - \wt{w}_{\pi(i)})\wt{G}_{\pi(i)}}_F \leq 2(k + 1)\eps \,.
\]
On the other hand we can write
\begin{equation*}
\begin{split}
& \Norm{\sum_{i = 1}^k (w_i^{2/3} - \wt{w}_{\pi(i)})\wt{G}_{\pi(i)}}_F   \geq \Norm{\sum_{i = 1}^k(w_i^{2/3} 
- \wt{w}_{\pi(i)}) w_i^{1/3}G_{\calL_i, 2s}}_F \hspace{0.2in} - \eps \sum_{i = 1}^k |w_i^{2/3} - \wt{w}_{\pi(i)}| \,.
\end{split}
\end{equation*}
Now for any coefficients $c_1, \dots , c_k$, we have
\[
\begin{split}
\norm{c_1G_{\calL_1, 2s} + \dots + c_k G_{\calL_k, 2s}}_F \geq  \frac{\gamma(|c_1| + \dots + |c_k|)}{\sqrt{k}}  
\end{split}
\]
where we used the joint non-degeneracy assumption.  Thus,
\begin{equation*}
\begin{split}
\Norm{\sum_{i = 1}^k (w_i^{2/3} - \wt{w}_{\pi(i)})\wt{G}_{\pi(i)}}_F  
& \geq \frac{\gamma w_{\min}^{1/3} \sum_{i = 1}^k |w_i^{2/3} - \wt{w}_{\pi(i)}|}{\sqrt{k}} - \eps \sum_{i = 1}^k |w_i^{2/3} - \wt{w}_{\pi(i)}| \\ & \geq  \left(\frac{\gamma w_{\min}^{1/3}}{\sqrt{k}} - \eps \right) \max_{i} (|w_i^{2/3} - \wt{w_{\pi(i)}}| )\,. 
\end{split}
\end{equation*}
Combining this with the previous inequality gives the desired bound.
\end{proof}

As a corollary to the above two theorems, the estimates $\wh{G}_i$ computed in Algorithm~\ref{algo:learning-mixture} are actually good estimates for the true individual Markov parameters $G_{\calL_i, 2s}$.  Now, running a stable variant of Ho-Kalman~\cite{oymak2019non} on the individual block Henkel matrices suffices to obtain estimates $\wh{A}_i, \wh{B}_i $, $\wh{C}_i, \wh{D}_i$. Formally,

\begin{theorem}[Stable Ho-Kalman,~\cite{oymak2019non}]
\label{thm:stable-ho-kalman}
For observability and controllability matrices that are rank $n$, the Ho-Kalman algorithm applied to $\hat{G}$ produces estimates $\hat{A},\hat{B}$, and $\hat{C}$ such that 
there exists similarity transform $T \in \R^{n \times n}$ such that 
\begin{equation*}
\max\{\|C - \hat{C} T\|_F, \| B - T^{-1} \hat{B}\|_F\} \leq 5\sqrt{n \| G - \hat{G}\|}
\end{equation*}
 and 
 \[\|A - T^{-1} \hat{A}T\|_F \leq \frac{\sqrt{n \|G - \hat{G}\|} \| H\| }{\sigma_{min}^{3/2}(H^{-})}\]
 and 
 \[\Norm{D - \hat{D}}_F \leq \sqrt{n}\Norm{G - \hat{G}}\]
 where in the above
 \[
 G = [D, CB, CAB, ..., CA^{2s-1}B] 
 \]
 and $H$ is the Hankel matrix constructed with the true parameters $G$.
\end{theorem}

Putting together the above theorems, we can prove our main result.

\begin{proof}[Proof of Theorem~\ref{thm:main-thm}]
The proof follows from simply combining the theorems above (rescaling $\eps$ appropriately by a polynomial in the other parameters).  Note that for each $i \in [k]$, the Hankel matrix $H_i$ with the true parameters, constructed in the Ho-Kalman algorithm satisfies $\| H_i\| \leq \sigma_{max}(\mathcal{O}_{\calL_i,s})\sigma_{max}(\mathcal{Q}_{\calL_i, s}) \leq  \poly(\kappa, s)$.  We also have $\sigma_{min}(H_i^{-}) \geq \sigma_{min}(\mathcal{O}_{\calL_i, s}) \sigma_{min}(\mathcal{Q}_{\calL_i, s}) \geq 1/\poly(\kappa)$ (see Claim~\ref{claim:eigenvalue-bound} and Claim~\ref{claim:boundA}).  Thus, we can indeed apply Theorem~\ref{thm:stable-ho-kalman}.  It is clear that the running time is a fixed polynomial in the number of samples $N$, once $$N \geq \poly\Paren{ m, n, p, s, \kappa, 1/w_{\min} , 1/\eps, 1/\gamma , 1/\delta }.$$ 
\end{proof}

\subsection{Bayes Optimal Clustering}

We are also able to show that our parameter learning algorithm actually allows us to do nearly Bayes-optimal clustering in the fully observed case i.e. when $C_i = I$ for all $i \in [k]$ \footnote{We believe that our clustering result naturally generalizes to the partially observed setting as long as assume that all of the $\calL_i = \calL(A_i,B_i,C_i,D_i)$ are written in their balanced realization (see \cite{oymak2019non} for a formal definition) which is just a canonical choice of the similarity transformation $U_i$ that is allowed to act on $A_i,B_i,C_i$}.

\begin{theorem}[Bayes-Optimal Clustering]\label{thm:optimal-clustering}
Let $\calM = w_1\calL_1 + \dots + w_k \calL_k$ be a mixture of LDS where each $\calL_i = \calL(A_i,B_i,C_i,D_i)$ with $C_i = I$ and assume that the mixture $\calM$ satisfies the assumptions in Section~\ref{sec:formal-setup}.  Then given 
\[
N = \poly\Paren{m,n,p , s, \kappa, 1/w_{\min} ,  1/\gamma,  1/\delta}
\]
sample trajectories from this mixture, there is an algorithm that runs in $\poly(N)$ time and has the following guarantees with probability $1 - \delta$.  There is a fixed permutation $\pi$ on $[k]$ such that given any trajectory $(u_1, \dots , u_l, y_1, \dots , y_l)$ with $l \leq O(s)$ and $\norm{u_i}, \norm{y_i} \leq \poly(m,n,p,s, \kappa, 1/w_{\min} ,  1/\gamma, 1/\delta)$ it computes a posterior distribution $(p_1 , \dots , p_k)$ on $[k]$ (with $p_1 + \dots + p_k = 1$) such that $(p_{\pi(1)}, \dots , p_{\pi(k)})$ is $\delta$-close in TV distance to the posterior distribution on $\calL_1, \dots , \calL_k$ from which the trajectory $(u_1, \dots , u_l, y_1, \dots , y_l)$ was drawn.  
\end{theorem}
\begin{remark}
Note that the condition that $\norm{u_i}, \norm{y_i} \leq \poly(m,n,p,s, \kappa, 1/w_{\min} ,  1/\gamma, 1/\delta)$ is satisfied with exponentially small failure probability for a random trajectory from any of the components since $l \leq O(s)$ and we have Claim~\ref{claim:boundA} which bounds the norm of powers of $A$ up to $A^{s}$.  The trajectories used in the learning algorithm have length $O(s)$ so in particular, we can nearly-optimally cluster those.
\end{remark}

\begin{proof}[Proof of Theorem~\ref{thm:optimal-clustering}]
We apply Theorem~\ref{thm:main-thm} with $\eps$ set as a sufficiently small inverse polynomial in the other parameters.  Because $C_i = I$, we can eliminate the similarity transformations $U_i$ and also without loss of generality the permutation $\pi$ on $[k]$ is the identity so we have
\[
\max_{i \in [k]}\left( \norm{A_i - \wh{A}_i }, \norm{B_i - \wh{B}_i }, \norm{D_i - \wh{D}_i }, |w_i - \wh{w}_i|  \right) \leq \eps \,.
\]
Now fix a choice of $i \in [k]$.  Define $\calP_i$ to be the probability that $(u_1, \dots u_l, y_1, \dots , y_l)$ is sampled from the LDS $\calL(A_i,B_i, I ,D_i)$ and let $\wh{\calP}_i$ be the probability that it is sampled from the LDS $\wh{\calL}_i = \calL(\wh{A}_i, \wh{B}_i, I , \wh{D}_i)$.  We can explicitly compute $\wh{\calP}_i$ from $\wh{A}_i, \wh{B}_i,  \wh{D}_i$ using regression.  Now we will bound the ratio $\calP_i/\wh{\calP}_i$ and prove that it is close to $1$.  We can write $\calP_i$ as an integral over all possibilities for $x_1, \dots , x_l$.  Now we explicitly compute the likelihood of $(u_1, \dots u_l, y_1, \dots , y_l, x_1, \dots , x_l)$.  To do this, note that fixing $(u_1, \dots u_l, y_1, \dots , y_l, x_1, \dots , x_l)$ uniquely determines $w_1, \dots , w_l, z_1, \dots , z_l$ as 
\[
\begin{split}
w_i &= x_{t+1} - A_ix_t - B_iu_t \\
z_i &= y_t - C_ix_t - D_iu_t
\end{split}
\]
and also the distributions of $w_i,z_i$ are standard Gaussian.  Thus, the probability of the $i$th LDS generating $(u_1, \dots u_l, y_1, \dots , y_l, x_1, \dots , x_l)$ is
\[
C\exp\left( -\frac{1}{2}\left( \sum_{t= 1}^{l-1} \norm{x_{t+1} - A_ix_t - B_iu_t }^2 + \sum_{t = 1}^l \norm{y_t - x_t  - D_iu_t}^2 + \sum_{t = 1}^l \norm{u_t}^2 \right) \right) 
\]
where $C$ is an appropriate normalizing constant obtained from the standard normal.  The formula is the same for $\wh{\calP}_i$ except with $A_i,B_i,D_i$ replaced with $\wh{A}_i, \wh{B}_i, \wh{D}_i$.  As long as 
\[
\norm{x_1}, \dots , \norm{x_l} \leq \poly(m,n,p,s, \kappa, 1/w_{\min} ,  1/\gamma, 1/\delta)
\]
then the ratio between the two likelihoods is in the interval $[1 - \sqrt{\eps}, 1 + \sqrt{\eps}]$ as long as $\eps$ was chosen sufficiently small initially.  However, the above happens with exponentially small failure probability for both $\calL_i$ and $\wh{\calL}_i$ so we actually have
\[
1 - 2\sqrt{\eps} \leq \frac{\wh{\calP}_i}{\calP_i} \leq 1 + 2\sqrt{\eps} \,.
\]
Combining the above over all $i \in [k]$ immediately implies the desired statement about the posterior distribution.
\end{proof}


\bibliography{local,scholar}

\newcommand{\etalchar}[1]{$^{#1}$}
\begin{thebibliography}{SMT{\etalchar{+}}18b}

\bibitem[{\AA}E71]{aastrom1971system}
Karl~Johan {\AA}str{\"o}m and Peter Eykhoff.
\newblock System identification—a survey.
\newblock {\em Automatica}, 7(2):123--162, 1971.

\bibitem[AGH{\etalchar{+}}14]{anandkumar2014tensor}
Animashree Anandkumar, Rong Ge, Daniel Hsu, Sham~M Kakade, and Matus Telgarsky.
\newblock Tensor decompositions for learning latent variable models.
\newblock {\em Journal of machine learning research}, 15:2773--2832, 2014.

\bibitem[AK05]{arora2005learning}
Sanjeev Arora and Ravi Kannan.
\newblock Learning mixtures of separated nonspherical gaussians.
\newblock 2005.

\bibitem[AM05]{achlioptas2005spectral}
Dimitris Achlioptas and Frank McSherry.
\newblock On spectral learning of mixtures of distributions.
\newblock In {\em International Conference on Computational Learning Theory},
  pages 458--469. Springer, 2005.

\bibitem[BDJ{\etalchar{+}}22]{bakshi2022robustly}
Ainesh Bakshi, Ilias Diakonikolas, He~Jia, Daniel~M Kane, Pravesh~K Kothari,
  and Santosh~S Vempala.
\newblock Robustly learning mixtures of k arbitrary gaussians.
\newblock In {\em Proceedings of the 54th Annual ACM SIGACT Symposium on Theory
  of Computing}, pages 1234--1247, 2022.

\bibitem[BK20]{bakshi2020outlier}
Ainesh Bakshi and Pravesh Kothari.
\newblock Outlier-robust clustering of non-spherical mixtures.
\newblock {\em arXiv preprint arXiv:2005.02970}, 2020.

\bibitem[BLMY23]{bakshi2023new}
Ainesh Bakshi, Allen Liu, Ankur Moitra, and Morris Yau.
\newblock A new approach to learning linear dynamical systems.
\newblock In {\em Proceedings of the 55th Annual ACM Symposium on Theory of
  Computing}, pages 335--348, 2023.

\bibitem[BS10]{belkin2010polynomial}
Mikhail Belkin and Kaushik Sinha.
\newblock Polynomial learning of distribution families.
\newblock In {\em 2010 IEEE 51st Annual Symposium on Foundations of Computer
  Science}, pages 103--112. IEEE, 2010.

\bibitem[BTBC16]{bulteel2016clustering}
Kirsten Bulteel, Francis Tuerlinckx, Annette Brose, and Eva Ceulemans.
\newblock Clustering vector autoregressive models: Capturing qualitative
  differences in within-person dynamics.
\newblock {\em Frontiers in Psychology}, 7:1540, 2016.

\bibitem[BV08]{brubaker2008isotropic}
S~Charles Brubaker and Santosh~S Vempala.
\newblock Isotropic pca and affine-invariant clustering.
\newblock {\em Building Bridges: Between Mathematics and Computer Science},
  pages 241--281, 2008.

\bibitem[CFG14]{candes2014towards}
Emmanuel~J Cand{\`e}s and Carlos Fernandez-Granda.
\newblock Towards a mathematical theory of super-resolution.
\newblock {\em Communications on pure and applied Mathematics}, 67(6):906--956,
  2014.

\bibitem[CLS19]{CLSmixedregression19}
Sitan Chen, Jerry Li, and Zhao Song.
\newblock Learning mixtures of linear regressions in subexponential time via
  fourier moments.
\newblock {\em CoRR}, abs/1912.07629, 2019.

\bibitem[CM21]{chen2021algorithmic}
Sitan Chen and Ankur Moitra.
\newblock Algorithmic foundations for the diffraction limit.
\newblock In {\em Proceedings of the 53rd Annual ACM SIGACT Symposium on Theory
  of Computing}, pages 490--503, 2021.

\bibitem[CP22]{chen2022learning}
Yanxi Chen and H~Vincent Poor.
\newblock Learning mixtures of linear dynamical systems.
\newblock In {\em International Conference on Machine Learning}, pages
  3507--3557. PMLR, 2022.

\bibitem[CYC13]{CYCmixedregression13}
Yudong Chen, Xinyang Yi, and Constantine Caramanis.
\newblock A convex formulation for mixed regression with two components:
  Minimax optimal rates, 2013.

\bibitem[Das99]{dasgupta1999learning}
Sanjoy Dasgupta.
\newblock Learning mixtures of gaussians.
\newblock In {\em 40th Annual Symposium on Foundations of Computer Science
  (Cat. No. 99CB37039)}, pages 634--644. IEEE, 1999.

\bibitem[DEF{\etalchar{+}}21]{DEF+mixedregression21}
Theo Diamandis, Yonina~C. Eldar, Alireza Fallah, Farzan Farnia, and Asuman
  Ozdaglar.
\newblock A wasserstein minimax framework for mixed linear regression, 2021.

\bibitem[Din13]{ding2013}
Feng Ding.
\newblock Two-stage least squares based iterative estimation algorithm for
  cararma system modeling.
\newblock {\em Applied Mathematical Modelling}, 37(7):4798--4808, 2013.

\bibitem[FJK96]{frieze1996learning}
Alan Frieze, Mark Jerrum, and Ravi Kannan.
\newblock Learning linear transformations.
\newblock In {\em Proceedings of 37th Conference on Foundations of Computer
  Science}, pages 359--368. IEEE, 1996.

\bibitem[FSJW08]{fox2008nonparametric}
Emily Fox, Erik Sudderth, Michael Jordan, and Alan Willsky.
\newblock Nonparametric bayesian learning of switching linear dynamical
  systems.
\newblock {\em Advances in neural information processing systems}, 21, 2008.

\bibitem[FTM17]{faradonbeh2017}
Mohamad Kazem~Shirani Faradonbeh, Ambuj Tewari, and George Michailidis.
\newblock Finite time identification in unstable linear systems.
\newblock {\em CoRR}, abs/1710.01852, 2017.

\bibitem[FTM18]{faradonbeh2018finite}
Mohamad Kazem~Shirani Faradonbeh, Ambuj Tewari, and George Michailidis.
\newblock Finite time identification in unstable linear systems.
\newblock {\em Automatica}, 96:342--353, 2018.

\bibitem[Gal16]{Galrinho2016LeastSM}
Miguel Galrinho.
\newblock Least squares methods for system identification of structured models.
\newblock 2016.

\bibitem[GCLF18]{gonze2018microbial}
Didier Gonze, Katharine~Z Coyte, Leo Lahti, and Karoline Faust.
\newblock Microbial communities as dynamical systems.
\newblock {\em Current opinion in microbiology}, 44:41--49, 2018.

\bibitem[GHK15]{ge2015learning}
Rong Ge, Qingqing Huang, and Sham~M Kakade.
\newblock Learning mixtures of gaussians in high dimensions.
\newblock In {\em Proceedings of the forty-seventh annual ACM symposium on
  Theory of computing}, pages 761--770, 2015.

\bibitem[GVX14]{goyal2014fourier}
Navin Goyal, Santosh Vempala, and Ying Xiao.
\newblock Fourier pca and robust tensor decomposition.
\newblock In {\em Proceedings of the forty-sixth annual ACM symposium on Theory
  of computing}, pages 584--593, 2014.

\bibitem[HK66]{ho1966effective}
BL~HO and Rudolf~E K{\'a}lm{\'a}n.
\newblock Effective construction of linear state-variable models from
  input/output functions.
\newblock {\em at-Automatisierungstechnik}, 14(1-12):545--548, 1966.

\bibitem[HK13]{hsu2013learning}
Daniel Hsu and Sham~M Kakade.
\newblock Learning mixtures of spherical gaussians: moment methods and spectral
  decompositions.
\newblock In {\em Proceedings of the 4th conference on Innovations in
  Theoretical Computer Science}, pages 11--20, 2013.

\bibitem[HK15]{huang2015super}
Qingqing Huang and Sham~M Kakade.
\newblock Super-resolution off the grid.
\newblock {\em Advances in Neural Information Processing Systems}, 28, 2015.

\bibitem[HLS{\etalchar{+}}18]{hazan2018spectral}
Elad Hazan, Holden Lee, Karan Singh, Cyril Zhang, and Yi~Zhang.
\newblock Spectral filtering for general linear dynamical systems.
\newblock {\em Advances in Neural Information Processing Systems}, 31, 2018.

\bibitem[HMR18]{hardt2018gradient}
Moritz Hardt, Tengyu Ma, and Benjamin Recht.
\newblock Gradient descent learns linear dynamical systems.
\newblock {\em Journal of Machine Learning Research}, 19:1--44, 2018.

\bibitem[HP15]{hardt2015tight}
Moritz Hardt and Eric Price.
\newblock Tight bounds for learning a mixture of two gaussians.
\newblock In {\em Proceedings of the forty-seventh annual ACM symposium on
  Theory of computing}, pages 753--760, 2015.

\bibitem[HSZ17]{hazan2017}
Elad Hazan, Karan Singh, and Cyril Zhang.
\newblock Learning linear dynamical systems via spectral filtering.
\newblock {\em CoRR}, abs/1711.00946, 2017.

\bibitem[HVBL17]{hallac2017toeplitz}
David Hallac, Sagar Vare, Stephen Boyd, and Jure Leskovec.
\newblock Toeplitz inverse covariance-based clustering of multivariate time
  series data.
\newblock In {\em Proceedings of the 23rd ACM SIGKDD international conference
  on knowledge discovery and data mining}, pages 215--223, 2017.

\bibitem[KHC20]{KHCmixedregression20}
Jeongyeol Kwon, Nhat Ho, and Constantine Caramanis.
\newblock On the minimax optimality of the em algorithm for learning
  two-component mixed linear regression, 2020.

\bibitem[KMS16]{kalliovirta2016gaussian}
Leena Kalliovirta, Mika Meitz, and Pentti Saikkonen.
\newblock Gaussian mixture vector autoregression.
\newblock {\em Journal of econometrics}, 192(2):485--498, 2016.

\bibitem[KMV10]{kalai2010efficiently}
Adam~Tauman Kalai, Ankur Moitra, and Gregory Valiant.
\newblock Efficiently learning mixtures of two gaussians.
\newblock In {\em Proceedings of the forty-second ACM symposium on Theory of
  computing}, pages 553--562, 2010.

\bibitem[KSKO20]{KSKO20metamixedregression}
Weihao Kong, Raghav Somani, Sham Kakade, and Sewoong Oh.
\newblock Robust meta-learning for mixed linear regression with small batches,
  2020.

\bibitem[KSS{\etalchar{+}}20]{KSS+20mixedregression}
Weihao Kong, Raghav Somani, Zhao Song, Sham Kakade, and Sewoong Oh.
\newblock Meta-learning for mixed linear regression, 2020.

\bibitem[Li00]{WongLi2000}
Wai Li.
\newblock On a mixture autoregressive model. j royal stat soc ser b.
\newblock {\em Journal of the Royal Statistical Society Series B}, 62:95--115,
  02 2000.

\bibitem[Lju98]{ljung1998system}
Lennart Ljung.
\newblock System identification.
\newblock In {\em Signal analysis and prediction}, pages 163--173. Springer,
  1998.

\bibitem[LL18]{LLmixedregression18}
Yuanzhi Li and Yingyu Liang.
\newblock Learning mixtures of linear regressions with nearly optimal
  complexity.
\newblock {\em CoRR}, abs/1802.07895, 2018.

\bibitem[LM22]{liu2022learning}
Allen Liu and Ankur Moitra.
\newblock Learning gmms with nearly optimal robustness guarantees.
\newblock In {\em Conference on Learning Theory}, pages 2815--2895. PMLR, 2022.

\bibitem[LM23]{liu2023learning}
Allen Liu and Ankur Moitra.
\newblock Robustly learning general mixtures of gaussians.
\newblock {\em J. ACM}, 70(3), may 2023.

\bibitem[MCY{\etalchar{+}}22]{mudrik2022decomposed}
Noga Mudrik, Yenho Chen, Eva Yezerets, Christopher~J Rozell, and Adam~S
  Charles.
\newblock Decomposed linear dynamical systems (dlds) for learning the latent
  components of neural dynamics.
\newblock {\em arXiv preprint arXiv:2206.02972}, 2022.

\bibitem[Moi15]{moitra2015super}
Ankur Moitra.
\newblock Super-resolution, extremal functions and the condition number of
  vandermonde matrices.
\newblock In {\em Proceedings of the forty-seventh annual ACM symposium on
  Theory of computing}, pages 821--830, 2015.

\bibitem[Moi18]{moitra2018algorithmic}
Ankur Moitra.
\newblock {\em Algorithmic aspects of machine learning}.
\newblock Cambridge University Press, 2018.

\bibitem[MR05]{mossel2005learning}
Elchanan Mossel and S{\'e}bastien Roch.
\newblock Learning nonsingular phylogenies and hidden markov models.
\newblock In {\em Proceedings of the thirty-seventh annual ACM symposium on
  Theory of computing}, pages 366--375, 2005.

\bibitem[MV10]{moitra2010settling}
Ankur Moitra and Gregory Valiant.
\newblock Settling the polynomial learnability of mixtures of gaussians.
\newblock In {\em 2010 IEEE 51st Annual Symposium on Foundations of Computer
  Science}, pages 93--102. IEEE, 2010.

\bibitem[MYP{\etalchar{+}}09]{mezer2009cluster}
Aviv Mezer, Yossi Yovel, Ofer Pasternak, Tali Gorfine, and Yaniv Assaf.
\newblock Cluster analysis of resting-state fmri time series.
\newblock {\em Neuroimage}, 45(4):1117--1125, 2009.

\bibitem[OO19]{oymak2019non}
Samet Oymak and Necmiye Ozay.
\newblock Non-asymptotic identification of lti systems from a single
  trajectory.
\newblock In {\em 2019 American control conference (ACC)}, pages 5655--5661.
  IEEE, 2019.

\bibitem[SBR19]{simchowitz2019learning}
Max Simchowitz, Ross Boczar, and Benjamin Recht.
\newblock Learning linear dynamical systems with semi-parametric least squares.
\newblock In {\em Conference on Learning Theory}, pages 2714--2802. PMLR, 2019.

\bibitem[SBTR12]{shah2012linear}
Parikshit Shah, Badri~Narayan Bhaskar, Gongguo Tang, and Benjamin Recht.
\newblock Linear system identification via atomic norm regularization.
\newblock In {\em 2012 IEEE 51st IEEE conference on decision and control
  (CDC)}, pages 6265--6270. IEEE, 2012.

\bibitem[SMT{\etalchar{+}}18a]{simchowitz2018learning}
Max Simchowitz, Horia Mania, Stephen Tu, Michael~I Jordan, and Benjamin Recht.
\newblock Learning without mixing: Towards a sharp analysis of linear system
  identification.
\newblock In {\em Conference On Learning Theory}, pages 439--473. PMLR, 2018.

\bibitem[SMT{\etalchar{+}}18b]{simchowitz2018}
Max Simchowitz, Horia Mania, Stephen Tu, Michael~I. Jordan, and Benjamin Recht.
\newblock Learning without mixing: Towards {A} sharp analysis of linear system
  identification.
\newblock {\em CoRR}, abs/1802.08334, 2018.

\bibitem[SPL05]{spinelli2005}
W.~Spinelli, L.~Piroddi, and M.~Lovera.
\newblock On the role of prefiltering in nonlinear system identification.
\newblock {\em IEEE Transactions on Automatic Control}, 50(10):1597--1602,
  2005.

\bibitem[SR19]{sarkar2019near}
Tuhin Sarkar and Alexander Rakhlin.
\newblock Near optimal finite time identification of arbitrary linear dynamical
  systems.
\newblock In {\em International Conference on Machine Learning}, pages
  5610--5618. PMLR, 2019.

\bibitem[SRD19]{sarkar2019nonparametric}
Tuhin Sarkar, Alexander Rakhlin, and Munther~A Dahleh.
\newblock Nonparametric finite time lti system identification.
\newblock {\em arXiv preprint arXiv:1902.01848}, 2019.

\bibitem[TP19]{tsiamis2019finite}
Anastasios Tsiamis and George~J Pappas.
\newblock Finite sample analysis of stochastic system identification.
\newblock In {\em 2019 IEEE 58th Conference on Decision and Control (CDC)},
  pages 3648--3654. IEEE, 2019.

\bibitem[VW04]{vempala2004spectral}
Santosh Vempala and Grant Wang.
\newblock A spectral algorithm for learning mixture models.
\newblock {\em Journal of Computer and System Sciences}, 68(4):841--860, 2004.

\bibitem[YCS13]{CYSmixedregression13}
Xinyang Yi, Constantine Caramanis, and Sujay Sanghavi.
\newblock Alternating minimization for mixed linear regression, 2013.

\bibitem[Zha11]{zhang2011}
Yong Zhang.
\newblock Unbiased identification of a class of multi-input single-output
  systems with correlated disturbances using bias compensation methods.
\newblock {\em Mathematical and Computer Modelling}, 53(9):1810--1819, 2011.

\end{thebibliography}
\bibliographystyle{alpha}

\newpage
\appendix
\section{Appendix}

\subsection{Jennrich's Algorithm}\label{sec:jennrichalg}

Jennrich's Algorithm is an algorithm for decomposing a tensor, say $T= \sum_{i=1}^r (x_i \otimes y_i \otimes z_i)$, into its rank-$1$ components that works when the fibers of the rank $1$ components i.e. $x_1, \dots , x_r$ are linearly independent (and similar for $y_1, \dots , y_r$ and $z_1, \dots , z_r$).

\begin{mdframed}
  \begin{algorithm}[Jennrich's Algorithm]
    \label{algo:jennrich}\mbox{}
    \begin{description}
    \item[Input:] Tensor $T' \in \R^{n \times n \times n}$ where 
$T' = T + E$
for some rank-$r$ tensor $T$ and error matrix $E$.  
    \item[Operation:]\mbox{}
\begin{enumerate}
\item Choose unit vectors $a,b \in \R^n$ uniformly at random
\item Let $T^{(a)}, T^{(b)}$ be $n \times n$ matrices defined as 
\begin{align*}
T^{(a)}_{ij} = T'_{i,j , \cdot } \cdot a \\
T^{(b)}_{ij} = T'_{i,j , \cdot } \cdot b
\end{align*}
\item Let $T_r^{(a)}, T_r^{(b)}$ be obtained by taking the top $r$ principal components of $T^{(a)}, T^{(b)}$ respectively.
\item Compute the eigendecompositions of $U = T_r^{(a)}(T_r^{(b)})^\dagger$ and $V = \left((T_r^{(a)})^\dagger T_r^{(b)}\right)^T$ (where for a matrix $M$, $M^\dagger$ denotes the pseudoinverse)
\item Let $u_1, \dots , u_r, v_1, \dots , v_r$ be the eigenvectors computed in the previous step.  
\item Permute the $v_i$ so that for each pair $(u_i,v_i)$, the corresponding eigenvalues are (approximately) reciprocals.
\item Solve the following for the vectors $w_i$
\[
\arg\min \Norm{T' - \sum_{i=1}^r u_i \otimes v_i \otimes w_i}_2^2
\]
\end{enumerate}
\item[Output:] the rank-$1$ components $\{u_i \otimes v_i \otimes w_i\}_{i=1}^r$ 
    \end{description}
  \end{algorithm}
\end{mdframed}

Moitra \cite{moitra2018algorithmic} gives a complete analysis of {\sc Jennrich's Algorithm}.  The result that we need is that as the error $E$ goes to $0$ at an inverse-polynomial rate, {\sc Jennrich's Algorithm} recovers the individual rank-$1$ components to within any desired inverse-polynomial accuracy.  
\begin{theorem}[\cite{moitra2018algorithmic}]\label{thm:robust-jennrich}
Let 
\[
T = \sum_{i=1}^r \sigma_i (x_i \otimes y_i \otimes z_i)
\]
where the $x_i,y_i,z_i$ are unit vectors and $\sigma_1 \geq \dots \geq \sigma_r > 0$.  Assume that the smallest singular value of the matrix with columns given by $x_1, \dots, x_r$ is at least $c$ and similar for the $y_i$ and $z_i$.  Then for any constant $d$, there exists a polynomial $P$ such that if
\[
\norm{E}_2 \leq  \frac{\sigma_1}{P(n, \frac{1}{c}, \frac{\sigma_1}{\sigma_r})}
\]
then with $1 - \frac{1}{(10n)^d}$ probability, there is a permutation $\pi$ such that the outputs of {\sc Jennrich's Algorithm} satisfy
\[
\norm{\sigma_{\pi(i)} (x_{\pi(i)} \otimes y_{\pi(i)} \otimes z_{\pi(i)}) - u_i \otimes v_i \otimes w_i}_2 \leq \sigma_1\left(\frac{\sigma_r c}{10\sigma_1 n}\right)^d
\]
for all $1 \leq i \leq r$.
\end{theorem}
\begin{remark}
Note that the extra factors of $\sigma_1$ in the theorem above are simply to deal with the scaling of the tensor $T$.  
\end{remark}

\end{document}